%% file: counterfactual-normalization.tex
\documentclass[letterpaper]{article}
\usepackage{uai2018}
\usepackage[margin=1in]{geometry}
\usepackage{natbib}

\usepackage{times}

\usepackage{microtype}
\usepackage{graphicx}
\usepackage{subfigure}
\usepackage{booktabs} 
\usepackage{amsmath}
\usepackage{amssymb}
\usepackage{amsthm}
\usepackage{bm}
\usepackage[ruled,vlined]{algorithm2e}
\usepackage{todonotes}

\newcommand{\ci}{\perp\!\!\!\perp}
\newtheorem{theorem}{Theorem}

\title{Counterfactual Normalization: Proactively Addressing Dataset Shift and Improving Reliability Using Causal Mechanisms}

\author{ {\bf Adarsh Subbaswamy} \\
Department of Computer Science \\
Johns Hopkins University\\
Baltimore, MD 21218 \\
\And
{\bf Suchi Saria}  \\
Department of Computer Science\\
Johns Hopkins University\\
Baltimore, MD 21218 \\
}

\begin{document}

\maketitle

\begin{abstract}
\input{sections/abstract.tex}
\end{abstract}

\section{INTRODUCTION}
\input{sections/introduction.tex}

\section{RELATED WORK}\label{sec:relwork}
\input{sections/related_work.tex}

\section{METHODS}\label{sec:cfn}

\input{sections/cf_normalization.tex}

\section{COMPLEXITY METRICS}\label{sec:complexity}
\input{sections/classification_complexity.tex}

\section{EXPERIMENTS}\label{sec:experiments}
\input{sections/experiments.tex}

\section{CONCLUSION}\label{sec:conclusion}
\input{sections/conclusion.tex}

\subsubsection*{Acknowledgements}
The authors would like to thank Katie Henry for her help in developing the sepsis classification DAG and Peter Schulam for suggesting experiment 5.1.1 and help clarifying presentation of the method.

\bibliography{references}
\bibliographystyle{apalike}

\newpage
\appendix

\input{sections/appendix.tex}

\end{document}

%% file: sections/abstract.tex
Predictive models can fail to generalize from training to deployment environments because of dataset shift, posing a threat to model reliability and the safety of downstream decisions made in practice. Instead of using samples from the target distribution to reactively correct dataset shift, we use graphical knowledge of the causal mechanisms relating variables in a prediction problem to proactively remove relationships that do not generalize across environments, even when these relationships may depend on unobserved variables (violations of the ``no unobserved confounders'' assumption). To accomplish this, we identify variables with unstable paths of statistical influence and remove them from the model. We also augment the causal graph with latent counterfactual variables that isolate unstable paths of statistical influence, allowing us to retain stable paths that would otherwise be removed. Our experiments demonstrate that models that remove vulnerable variables and use estimates of the latent variables transfer better, often outperforming in the target domain despite some accuracy loss in the training domain.

%% file: sections/introduction.tex
Classical supervised machine learning methods for prediction problems assume that training and test data are independently and identically distributed from a fixed distribution over the input features $\mathbf{X}$ and target output label $T$, $p(\mathbf{X}, T)$. When this assumption does not hold, training with classical frameworks can yield unreliable models and, in the case of safety-critical applications like medicine, dangerous predictions \citep{dyagilev2016learning,caruana2015intelligible,schulam2017NIPS}. Unreliable models may have performance that is not \emph{stable}---model performance varies greatly when the test distribution is different from the training distribution in scenarios where invariance to the underlying changes are desirable and expected. Unreliability arises because models are often deployed in dynamic environments that systematically differ from the one in which the historical training data was collected---a problem known as \emph{dataset shift} which results in poor generalization. Most existing methods for addressing dataset shift are reactive: they use unlabeled data from the target distribution during the learning process (see \citet{quionero2009dataset} for an overview). However, when the differences in environments are unknown prior to model deployment (e.g., no available data from the target environment), it is important to understand what aspects of the prediction problem can change and how we can train models that will be robust to these changes. In this work we consider this problem of \emph{proactively addressing dataset shift} for discriminative models.

To illustrate, we will consider diagnosis, a problem common to medical decision making. The goal is to detect the presence of a target condition $T$. The features used can be split into three categories: \emph{risk factors} for the target condition (causal antecedents), \emph{outcomes} or symptoms of the condition (causal descendents), and \emph{co-parents} that serve as alternative explanations for the observations (e.g., comorbidities and treatments). The causal mechanisms (directional knowledge of causes and effects, e.g., beta blockers lower blood pressure) relating variables in a prediction problem can be represented using directed acyclic graphs (DAGs), such as the one in Figure \ref{fig:dag}a. As an example (Figure \ref{fig:dag}b), a hospital may wish to screen for meningitis $T$, which can cause blood pressure (BP) $Y$ to drop dangerously low. Smoking $D$ is a risk factor for meningitis, and also causes heart disease for which patients are prescribed  beta blockers $C$ (a type of medication that lowers blood pressure). However, \emph{domain-dependent confounding} (Figure \ref{fig:dag}b) and \emph{selection bias} (Figure \ref{fig:dag}c) can cause certain distributions in the graph to change across domains, resulting in dataset shift.

Consider domain-dependent confounding in which relevant variables may be unobserved and distributions involving these variables may change across domains. In diagnosis, unobserved variables are likely to be risk factors (e.g., behavioral factors, genetics, and geography) that confound the relationship between the target condition and comorbidities/treatments. For example (Figure \ref{fig:dag}b), smoking ($D$) may not be recorded in the data, and the policy used to prescribe beta blockers to smokers $p(C|D)$ will vary between doctors and hospitals. When $D$ is observed, the changes in the prescription policy can be adjusted for. More generally, others have described solutions to ensuring model stability across environments with differences in policies \citep{schulam2017NIPS}. Specifically, they optimize the \emph{counterfactual} risk to explicitly account for variations in policy between train and test environments (e.g., \citet{swaminathan2015counterfactual,schulam2017NIPS}). However, this requires \emph{ignorability} assumptions (also known as the \emph{no unobserved confounders} assumption in causal inference), that may not hold in practice (such as when $D$ is not observed). Violations of this assumption have implications on model reliability. For example, in Figure \ref{fig:dag}b by $d$-separation \citep{koller2009probabilistic} $C$ has two active paths to $T$ when conditioned on $Y$: $C\leftarrow D \rightarrow T$ and $C\rightarrow Y \leftarrow T$. The first path is \emph{unstable} because it contains an edge $D\rightarrow C$ encoding the distribution that changes between environments $p(C|D)$. The second path, however, encodes medical effects that are stable---$p(Y|T,C)$ does not change. Naively including $C$ and $Y$ in the model will capture both paths, leaving the model \emph{vulnerable} to learning the relationship along the unstable path.

Similarly, selection bias (Figure \ref{fig:dag}c) adds auxiliary variables to the graph (i.e., $S$) which can create unstable paths that contribute to model unreliability. Certain subpopulations with respect to the target and comorbidities may be underrepresented in the training data ($S=1$). For example, patients without meningitis who take beta blockers ($T=0,C=1$) may be underrepresented because they rarely visit the hospital due to a local chronic care facility which helps them manage their chronic condition. This introduces a new unstable active path from $C$ to $T$: $C \rightarrow S \leftarrow T$. As before, the path through $Y$ remains stable. In the case of selection bias or  domain-dependent confounding, can we remove the influence of unstable paths while retaining the influence of stable paths?

We propose removing \emph{vulnerable} variables---variables with unstable active paths to the target--- from the conditioning set of a discriminative model in order to learn models that are stable to changes in environment. In Figure \ref{fig:dag}, this means we must remove $C$ from the model. In doing so, $Y$ becomes vulnerable as well because of the paths $Y\leftarrow C \leftarrow D \rightarrow T$ in \ref{fig:dag}b and $Y\leftarrow C \rightarrow S \leftarrow T$ in \ref{fig:dag}c, so we must remove $Y$. While this removes all unstable paths, it also removes stable paths (in fact, it removes all stable paths in this example). However, in certain situations we describe, we can retain some of the stable paths between the target vulnerable variables by considering counterfactual variables. In our example, if we somehow knew an adjusted counterfactual value of $Y$, denoted $Y(C=\emptyset)$---the value of $Y$ for which the effects of $C$ were removed (e.g., the blood pressure had the patient not been treated)---then this adjusted $Y$ would only contain the information along the stable path $T\rightarrow Y$. This concept is inspired by potential outcomes in causal inference and allows us to retain stable paths that would otherwise be removed along with the unstable paths.

\begin{figure}[!t]
\begin{center}
\centerline{\includegraphics[scale=0.28]{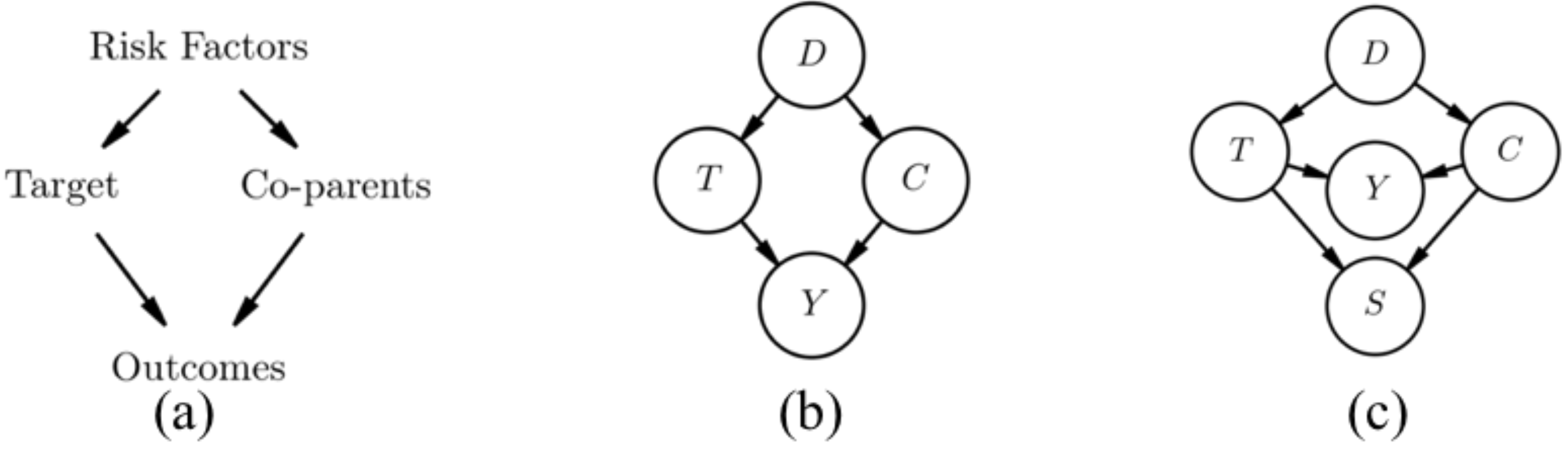}}
\caption{(a) General diagnosis DAG. (b) The DAG capturing causal mechanisms for the medical screening example. The features are blood pressure $Y$ and beta blockers $C$. The target label $T$ is meningitis. Smoking $D$ is unobserved. (c) Selection  bias $S$ is introduced.}
\label{fig:dag}
\end{center}
\vskip -0.4in
\end{figure}

\textbf{Contributions:} First, we identify variables which make a statistical model \emph{vulnerable} to learning unstable relationships that do not generalize across datasets (due to selection bias or unobserved domain-dependent confounding) which must be removed from a discriminative model for its performance to be stable. Second, we define a \emph{node-splitting} operation which modifies the DAG to contain interpretable latent counterfactual variables which isolate unstable paths allowing us to retain some stable paths involving vulnerable variables. By allowing unstable paths to depend on unobserved variables, we generalize previous works that learn stable models by assuming there are no unobserved confounders, intervening on the unstable policy, and predicting potential outcomes (see e.g., \cite{schulam2017NIPS}). Third, we provide algorithms for determining stable conditioning sets and which counterfactuals to estimate. Fourth, we explain how including the latent features can make a classification problem measurably simpler due to their reduced variance. In simulated and real data experiments we demonstrate that our method improves stability of model performance.

%% file: sections/related_work.tex
\textbf{Proactive and Reactive Approaches:}
Reactive predictive modeling methods for countering dataset shift typically require representative unlabeled samples from the test distribution \citep{storkey2009training}. These methods work by re-weighting the training data or extracting transferable features  (e.g., \citet{shimodaira2000improving, gretton2009covariate, gong2016domain,zhang2013domain}). To proactively address perturbations of test distributions, recent work considers formal \emph{verification} methods for bounding the performance of trained models on perturbed inputs (e.g., \citet{raghunathan2018certified,dvijotham2018dual}). Complementary to this, others have developed methods based on distributional robustness for training models to be minimax optimal to perturbations of bounded magnitude in order to guard against adversarial attacks \citep{sinha2018certifying} and improve generalization \citep{rothenhausler2018anchor}. We consider the related problem of training models that are stable to arbitrary shifts in distribution.

Beyond predictive modeling, previous work has considered estimation of causal models in the presence of selection bias and confounding. For example, \citet{spirtes1995causal} learn the structure of the causal DAG from data affected by selection bias. Others have studied methods and conditions for \emph{identification} of causal effects under simultaneous selection and confounding bias (e.g., \citet{bareinboim2012controlling,bareinboim2015selection, correa2018generalized}). \citet{correa2017selectionidentification} determine conditions under which interventional distributions are identified without using external data.

\textbf{Transportability:}
The goal of an experiment is for the findings to generalize beyond a single study, a concept known as \emph{external validity} \citep{campbell1963experimental}. Similarly, in causal inference \emph{transportability}, formalized in \citet{pearl2011transportability}, transfers causal effect estimates from one environment to another. \citet{bareinboim2013meta} generalize this to transfer causal knowledge from multiple source domains to a single target domain. Rather than transfer causal estimates from source to target, the proposed method learns a single statistical model whose predictions should perform well on the source domain while also generalizing well to new domains.

\textbf{Graphical Representations of Counterfactuals:}
The node-splitting operation we introduce in Section \ref{sec:cfn} is similar to the node-splitting operation in Single World Intervention Graphs (SWIGs) \citep{SWIG}. However, intervening in a SWIG results in a generative graph for a potential outcome with the factual outcome removed from the graph. By contrast, our node-splitting operation yields a modified generative graph of the factual outcomes with new intermediate counterfactual variables.
Other graphical representations such as twin networks \citep{pearl2009} and counterfactual graphs \citep{shpitser2007counterfactuals} simultaneously represent factual and counterfactual outcomes, rather than the intermediate counterfactuals exploited in this work.

%% file: sections/cf_normalization.tex
\subsection{BACKGROUND}
\input{sections/background.tex}

\subsection{COUNTERFACTUAL NORMALIZATION}
\subsubsection{Assumptions About Structure of the Graph}
Counterfactual Normalization uses a DAG, $\mathcal{G}$, that
represents the causal mechanisms relating variables in a prediction problem. Let $\mathbf{O}$ denote the observed variables, and $T\in\mathbf{O}$ be the target variable to predict ($T$ is unobserved in test distributions). We make no further assumptions about the edges relating observed variables. Let $ch(\cdot)$ and $pa(\cdot)$ represent children and parents in $\mathcal{G}$, respectively.

$\mathcal{G}$ can contain unobserved variables $\mathbf{U}$, which we will use to represent domain-dependent confounding. An unobserved variable must have at least two children so that it confounds the relationship between its children. Domain-dependent confounding occurs when $p(\mathbf{U}|pa(\mathbf{U}))$ or $p(ch(\mathbf{U})|\mathbf{U})$ changes across domains. $\mathcal{G}$ can also contain an additional variable $S$ which represents the selection mechanism that induces selection bias in the training data. The mechanism is given by $p(S=1|pa(S))$ where $pa(S)$ is assumed to be nonempty and $S$ is always assumed to be conditioned upon in the training domain.

\newlength{\textfloatsepsave} \setlength{\textfloatsepsave}{\textfloatsep} \setlength{\textfloatsep}{0.1in} 
\begin{algorithm}[!t]\label{alg:stable}
 \KwIn{Graph $\mathcal{G}$, number of variables $N$, observed variables $\mathbf{O}$, target $T$}
 \KwOut{Stable conditioning set $\mathbf{Z}$, Vulnerable set $\mathbf{V}$}
 $\mathbf{Z}=\mathbf{O}\setminus T$\;
 $\mathbf{V}=\emptyset$\;
 \For{$k=1$ to $N-1$}{
 Conditioned on $\mathbf{Z}$, find the set $\mathbf{A}$ of active paths starting with $T$ and ending at $v\in \mathbf{Z}$ of length $k$\;
 \For{active path $a\in\mathbf{A}$}{
    $v=$ last variable in $a$\;
    \If{$a$ is unstable}{
    $\mathbf{Z} = \mathbf{Z}\setminus v$\;
    $\mathbf{V} = \mathbf{V}\bigcup v$\;
    }
 }
 }
 \caption{Constructing a Stable Conditioning Set}
\end{algorithm} 

\subsubsection{Constructing a Stable Set}
The goal of Counterfactual Normalization is to find a set of observed variables and adjusted versions of observed variables that contains no active unstable paths while maximizing the number of active stable paths it contains. First, we define an \emph{unstable path} to be a path to the target $T$ that contains variables or edges which encode a distribution that can change across environments. These are edges involving unobserved variables $\mathbf{U}$ (domain-dependent confounding) or the selection mechanism variable $S$. Thus, an unstable path is a path to $T$ which contains $S$ or a variable in $\mathbf{U}$.

We can find a set, $\mathbf{Z}$, of observed variables with no active unstable paths using Algorithm \ref{alg:stable}, which considers active paths of increasing length that begin with $T$, and removes \emph{vulnerable} variables $\mathbf{V}$ reachable by unstable active paths. 
\begin{theorem}
Algorithm \ref{alg:stable} will result in a set $\mathbf{Z}$ that contains no unstable active paths to $T$.
\end{theorem}
\begin{proof}[Proof Sketch]
We show that on iteration $k$, removing a variable from $\mathbf{Z}$ does not create an active unstable path to a member of $\mathbf{Z}$ of length $\leq k$ (see supplement).
\end{proof}

We now consider expanding the stable conditioning set $\mathbf{Z}$ by including some variables in $\mathbf{V}$ or adjusted versions of these variables. The adjusted versions are counterfactuals which we place on a modified DAG $\mathcal{G}^*$ through a procedure called \emph{node-splitting}.

\setlength{\textfloatsep}{0.1in} 
\begin{algorithm}[!t]\label{alg:node}
\DontPrintSemicolon
 \KwIn{Graph $\mathcal{G}$, node $Y$, observed parents of $Y$ to intervene upon $\mathbf P$}
 \KwOut{Modified graph $\mathcal{G}^*$}
 1. Insert counterfactual node $Y(\mathbf{P}=\bm{\emptyset})$\;
 2. Delete edges $\{x\rightarrow Y: x\in pa(Y)\setminus \mathbf{P}\}$ \;
 3. Insert edges $\{x\rightarrow Y(\mathbf{P}=\bm{\emptyset}):x\in pa(Y)\setminus \mathbf{P}\}$\;
 4. Insert edge $Y(\mathbf{P}=\bm{\emptyset})\rightarrow Y$\;
 \caption{Node-splitting Operation}
\end{algorithm} 
\subsubsection{Node-Splitting}
Assume each variable $v\in\mathcal{G}$ has a corresponding structural equation in which it is a function of its parents and an exogenous, unobserved, and independent noise term: $v=f_v(pa(v),\varepsilon_v)$. We want to compute a counterfactual version of $v$ in which we remove the effects of (i.e., intervene upon) some of its parents. Denote the set of parents we intervene upon as $\mathbf{P}$. Given $v$'s factual value and the factual values of $\mathbf{P}$, we calculate the counterfactual value $v(\mathbf{P}=\emptyset)$ (remove the effects of parents in $\mathbf{P}$ by intervening and setting these parents to ``null''). In the diagnosis example of Figure \ref{fig:dag}b, an example counterfactual variable would be $Y(C=\emptyset)$: the patient's blood pressure if we removed the effects of the treatments they were given. Note that we must observe the factual value of parents we intervene on---they must be observed variables.

Removing the effects of only a subset of the parents requires being able to consider the effects of a parent while holding fixed the effects of the other parents of the variable. For this reason, we assume that the effects of parents on children are independent---they have no interactions. We specifically consider \emph{additive} structural equations which satisfy this requirement. Estimation of the counterfactuals requires fitting the relevant structural equations using the factual outcome data by maximum likelihood estimation. We can now define the node-splitting operation, which is given in Algorithm \ref{alg:node}. Given a variable and the subset of its parents to intervene upon, we set the intervened parents to ``null'' and place a latent counterfactual version of the variable onto the graph as a parent of its factual version. Unlike traditional SEM interventions, we retain the factual version of the parents we intervene on in the graph. The counterfactual version subsumes the parents (in the original graph $\mathcal{G}$) of its factual version that were not intervened upon. The modified graph $\mathcal{G}^*$ is an equivalent model of the factual data generating process.

The consequence of node-splitting is that while the factual version of a variable may be vulnerable, after intervening on some of its parents its counterfactual version may no longer be vulnerable. Consider a vulnerable variable $v$ which, if added to a stable conditioning set $\mathbf{Z}$, would yield at least one unstable active path to $T$. If the unstable path is of the form $v \leftarrow X \dots T$, then since $X$ is a parent of $v$ we can intervene on $X$. After node-splitting the new path would be $v(X=\emptyset)\rightarrow v \leftarrow X \dots T$. Since $v$ is not conditioned on, this collider path is blocked. Thus, this unstable path is not active for $v(X=\emptyset)$ though it was active for $v$. However, if the unstable path were of the form $v \rightarrow X \dots T$, then we cannot intervene on $X$ (not a parent of $v$) and any counterfactual version of $v$ will inherit the unstable active path: $v(\emptyset)\rightarrow v\rightarrow X \dots T$. The first case shows that for unstable paths from a vulnerable variable $v$ that begin through a parent of $v$, intervening on the parent yields a counterfactual in which these unstable paths are not active. The second case shows that unstable active paths from $v$ that begin through a child of $v$ cannot be removed by node-splitting. We can also intervene on a variable's observed parents that are not along unstable paths. As we discuss in Section \ref{sec:complexity}, the potential benefit is that counterfactual variables have reduced variance than their factual versions.

A question remains: does conditioning upon a stable counterfactual version of a vulnerable variable cause any unstable paths to become active? Conditioning on a variable can only open collider paths, so the only cases we must consider are when the counterfactual is a collider or descendant of a collider. In these cases, the active paths that meet at the collider are reachable by the counterfactual through at least one of its parents. However, we know that these paths are stable since the counterfactual is stable: we would have intervened on any parents which were along unstable paths. Thus, conditioning on a stable counterfactual does not activate any new unstable paths.

\setlength{\textfloatsep}{0.1in} 
\begin{algorithm}[!t]\label{alg:CFN}
 \KwIn{Graph $\mathcal{G}$, Vulnerable set in reverse topological order $\mathbf{V}$, Stable set $\mathbf{Z}$, Target $T$}
 \KwOut{Final conditioning set $\mathbf{Z'}$}
 $\mathbf{Z'}=\mathbf{Z}$\;
 \For{$v\in\mathbf{V}$}{
 \uIf{$v$ has no active stable paths w.r.t. $\mathbf{Z'}$}{pass\;}
 \uElseIf{$v$ has no active unstable paths w.r.t. $\mathbf{Z'}$}{$\mathbf{Z'}=\mathbf{Z'}\bigcup v$}
 \uElseIf{all unstable paths w.r.t. $\mathbf{Z'}$ from $v$ to $T$ are through observed parents $\mathbf{P}\subseteq pa(v)$ of $v$}{
 Node-split and modify $\mathcal{G}$\;
 $\mathbf{Z'}=\mathbf{Z'}\bigcup v(\mathbf{P}=\emptyset)$\;}
 \Else{pass\;}
 }
 Prune $\mathbf{Z'}$ of variables with no active stable paths.
 \caption{Retaining Vulnerable Variables}
\end{algorithm} 

\subsubsection{Adding to the Conditioning Set}
After finding a stable set $\mathbf{Z}$ of observed variables to condition upon, we must consider adding back each of the vulnerable variables that were removed. First, there may be variables with no unstable active paths because collider paths became inactive after these variables were removed from $\mathbf{Z}$. Second, we know that if a variable's active unstable paths go through observed parents, we can intervene on those parents, node-split in $\mathcal{G}^*$, and add the counterfactual version to the conditioning set. Because conditioning on the counterfactual may open stable paths involving its non-vulnerable parents, we want to make sure that non-vulnerable parents that may be in $\mathbf{V}$ (first case) are considered after the counterfactual. For this reason, we consider adding the variables in $\mathbf{V}$ to $\mathbf{Z}$ in reverse topological order. Algorithm \ref{alg:CFN} shows the procedure for adding variables to $\mathbf{Z}$. We condition on the resulting set, $\mathbf{Z'}$ and use it to predict $T$ by modeling $p(T|\mathbf{Z'})$. 
\begin{theorem}
Algorithm \ref{alg:CFN} does not activate any unstable paths and results in a stable set $\mathbf{Z'}$.
\end{theorem}
\begin{proof}[Proof Sketch]
We show all branches in the algorithm do not activate unstable paths (see supplement).
\end{proof}

\begin{figure}[!t]
\begin{center}
\centerline{\includegraphics[scale=0.25]{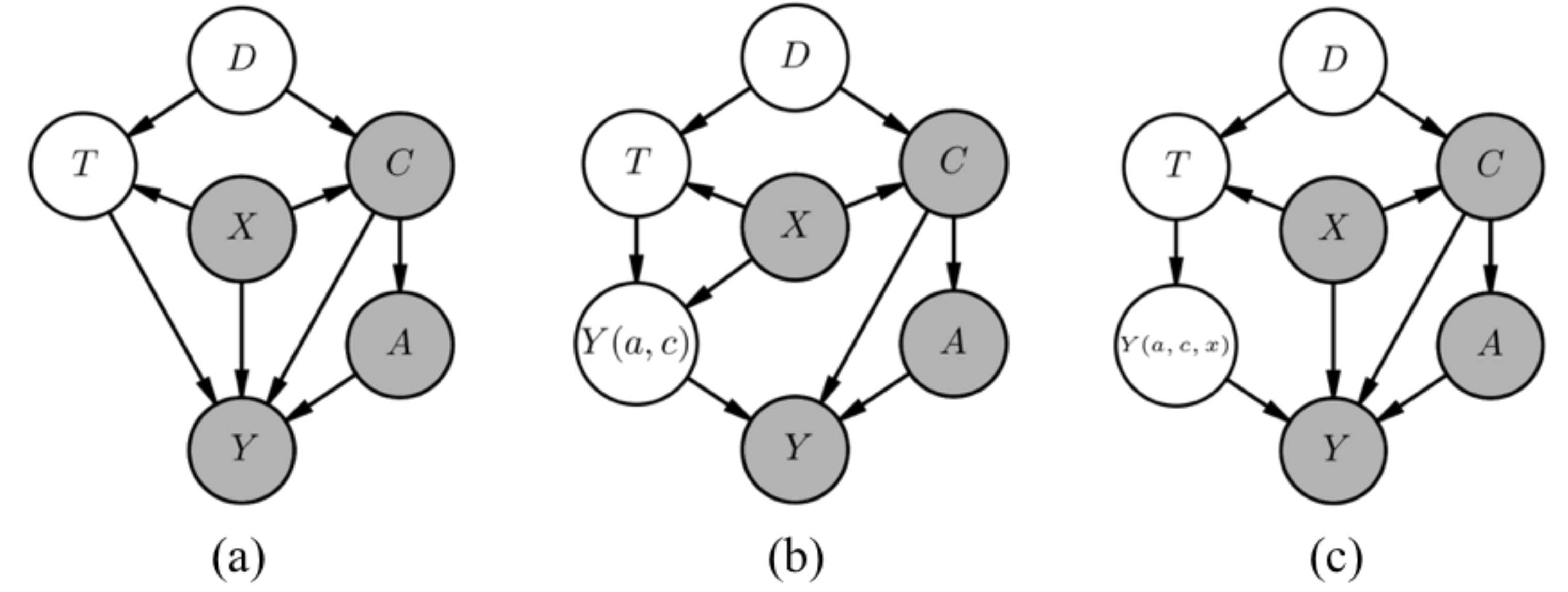}}
\caption{(a) The DAG of causal mechanisms for the expanded medical screening example. (b) The modified DAG after node-splitting yielding the latent signal value under no treatment or chronic condition $Y(a, c)$. (c) The modified DAG after node-splitting and additionally adjusting for other covariates $Y(a,c,x)$.}
\label{fig:cf-dag}            
\end{center}
\vskip -0.2in
\end{figure}
\subsubsection{An Example}
To illustrate node-splitting and Counterfactual Normalization, consider the expanded domain-dependent confounding diagnosis example in Figure \ref{fig:cf-dag}a. $C$ represents a chronic condition (e.g., heart disease), $A$ represents treatments (e.g., beta blockers), and $X$ represents age (a demographic risk factor). $D\in \mathbf{U}$ is an unobserved variable, and we allow $p(C|D)$ to vary across domains.

In finding a stable set $\mathbf{Z}$, we remove $C$ (unstable path of length 2), and then $A$ and $Y$ (unstable paths of length 3) which yields $\mathbf{Z}=\{X\}$. Now we consider the variables in $\mathbf{V}=\{C,Y,A\}$ in reverse topological order. $Y$ has unstable active paths through $A$ and $C$. Since they are observed variables, we intervene on them to generate the counterfactual $Y(C=\emptyset,A=\emptyset)$ in Figure \ref{fig:dag}b and add it to $\mathbf{Z'}$ after node-splitting. Now consider $A$, which has no stable active paths to $T$ so we do not add it to $\mathbf{Z'}$. Similarly, $C$ has no stable active paths. Thus, $\mathbf{Z'} = \{Y(C=\emptyset,A=\emptyset), X\}$ is the conditioning set we would use to predict $T$ by modeling $p(T|\mathbf{Z'})$.

%% file: sections/background.tex
\textbf{Potential Outcomes}

The proposed method involves the estimation of counterfactuals, which can be formalized using the Neyman-Rubin potential outcomes framework \citep{neyman1923,rubin1974}. For outcome variable $Y$ and intervention $A$, we denote the potential outcome by $Y(a)$: the value $Y$ would have if $A$ were observed to be $a$.

In general, the distributions $p(Y(a))$ and $p(Y|A=a)$ are not equal. For this reason, estimation of the distribution of the potential outcomes relies on two assumptions:

\textbf{Consistency}: The distribution of the potential outcome under the observed intervention is the same as the distribution of the observed outcome. This implies $p(Y(a)|A=a) = p(Y|A=a)$.

\textbf{Conditional Ignorability}: $Y(a) \ci A | X$, $\forall a \in A$. There are no unobserved confounders. This implies $p(Y(a)|X, A=a') = p(Y(a)|X, A=a)$.

\textbf{Counterfactuals and SEMS}

\citet{shpitser2008complete} develop a causal hierarchy consisting of three layers of increasing complexity: association, intervention, and counterfactual. Many works in causal inference are concerned with estimating average treatment effects---a task at the intervention layer because it uses information about the interventional distribution $p(Y(a)|X)$. In contrast, the proposed method requires counterfactual queries which use the distribution $p(Y(a)|Y, a', X)$ s.t. $a\not = a'$ \footnote{The distinction is that $p(Y(a)|X)$ reasons about the effects of causes while $p(Y(a)|Y, a', X)$ reasons about the causes of effects (see, e.g., \citet{pearl2015causes}).}. That is, given that we observed an individual's outcome to be $Y$ under intervention $a'$, what would the distribution of their outcome have been under a different intervention $a$?

In addition to the assumptions for estimating potential outcomes, computing counterfactual queries requires functional or structural knowledge \citep{pearl2009}. We can represent this knowledge using causal structural equation models (SEMs). These models assume variables $X_i$ are functions of their immediate parents in the generative causal DAG and exogenous noise $u_i$: $X_i = f_i(pa(X_i), u_i)$. Reasoning counterfactually at the level of an individual unit requires assumptions on the form of the functions $f_i$ and independence of the $u_i$, because typically we are interested in reasoning about interventions in which the exogenous noise variables remain fixed. We build on this to estimate the latent counterfactual variables introduced within the proposed procedure.

%% file: sections/classification_complexity.tex
Beyond removing unstable paths, what are other benefits of the proposed method? For binary prediction problems, the geometric complexity (on the basis of euclidean distance) of the class boundary of a dataset can decrease when using the latent counterfactual variables instead of the factual and vulnerable variables. This is similar to the work of \citet{alaa2017} who use the smoothness of the treated and untreated response surfaces to quantify the difficulty of a causal inference problem. To measure classifier-independent geometric complexity we will use two types of metrics developed by \citet{ho2000measuring,ho2002complexity}: measures of overlap of individual features and measures of separability of classes.

For measuring feature overlap, we use the maximum Fisher's discriminant ratio  of the features. For a single feature, this measures the spread of the means for each class ($\mu_1$ and $\mu_2$) relative to their variances ($\sigma_1^2$ and $\sigma_2^2$): $\frac{(\mu_1 - \mu_2)^2}{\sigma_1^2 + \sigma_2^2}$. Since the proposed method uses counterfactual variables in which we have removed the effects of some parents, this removes sources of variance in the variable. Thus, we expect the variances of each class to reduce resulting in increased feature separability and a corresponding increased Fisher's discriminant ratio. 

One measure of separability of classes is based off of a test \citep{friedman1979} for determining if two samples are from the same distribution. First, compute a minimum spanning tree (MST) that connects all the data points regardless of class. Then, the proportion of nodes which are connected to nodes of a different class is an approximate measure of the proportion of examples on the class boundary. Higher values of this proportion generally indicate a more complex boundary, and thus a more difficult classification problem. 

However, this metric is only sensitive to which class neighbors are closer, and not the relative magnitudes of intraclass and interclass distances.  Another measure of class separability is the ratio between the average intraclass nearest neighbor distance and the average interclass nearest neighbor distance. This measures the relative magnitudes of the dispersion within classes and the gap between classes. We expect intraclass distances to decrease because the data units are transformed to have the same value of the intervened parents, reducing sources of variance (e.g., less variance in counterfactual untreated BP than in factual BP).

The non-$T$ parents of a variable add variance to the prediction problem through their effects on children of $T$. By removing their effects from children of $T$, the proposed method can directly increase the signal-to-noise ratio of the classification problem. With respect to the geometric complexity of the class boundary, this manifests itself through reductions in the variance within a class, as we demonstrate in a simulated experiment.

%% file: sections/experiments.tex
We demonstrate that without requiring samples from the target distribution during training, Counterfactual Normalization results in discriminative models with more stable performance across datasets. In all experiments we train models using only source data and evaluate on test data from both the the source and target domains.

\subsection{SIMULATED EXPERIMENTS}
\subsubsection{Linear Gaussian Example}
\setlength{\textfloatsep}{\textfloatsepsave} 
\begin{figure}[!t]
\begin{center}
\centerline{\includegraphics[scale=0.43]{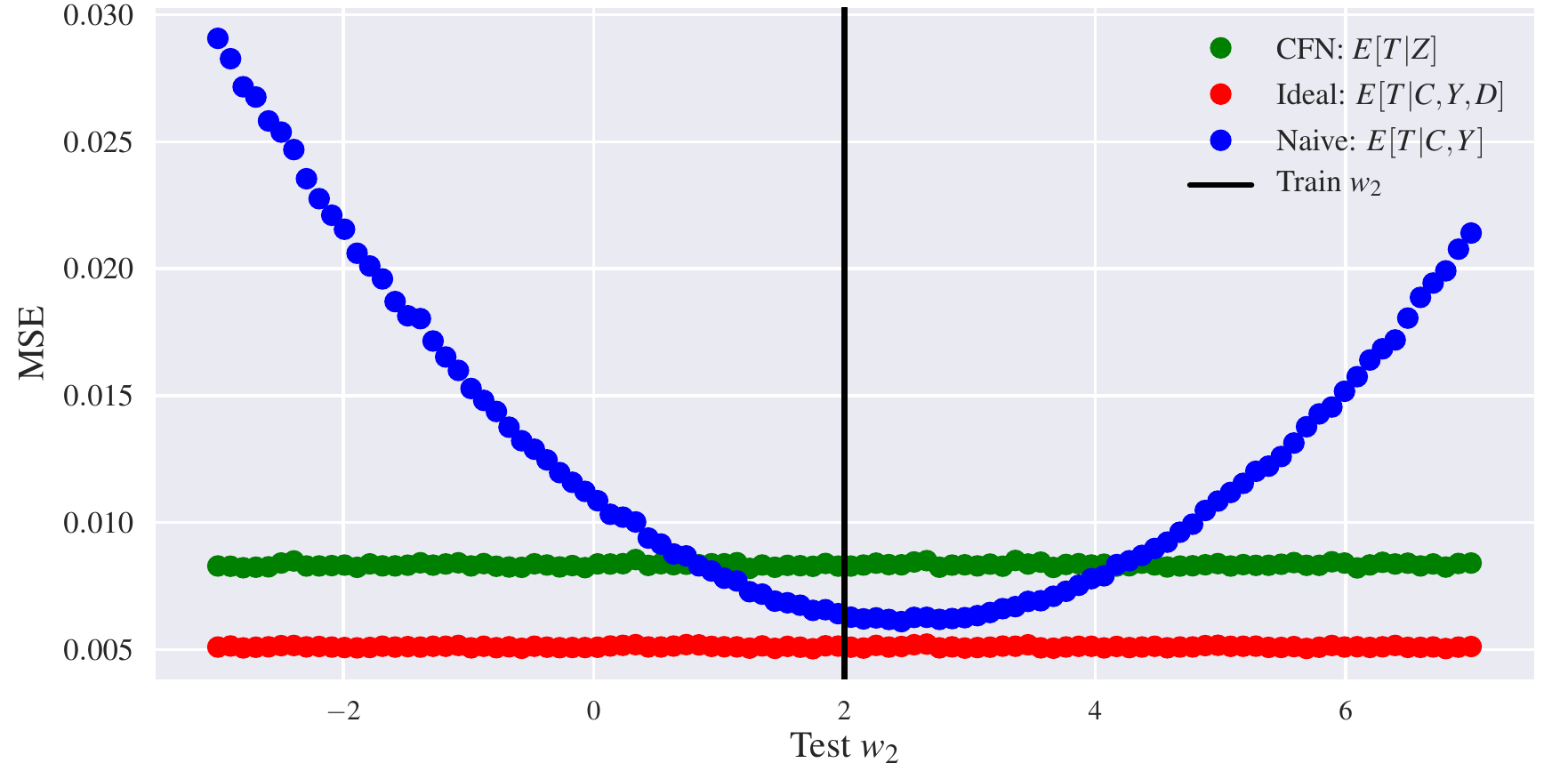}}
\caption{Test MSE as $w_2$ varies in the linear experiment. Performance of CFN and Ideal remains stable while the naive model's performance can arbitrarily worsen.}
\label{fig:lin-mse}
\end{center}
\vskip -0.4in
\end{figure}
We consider a regression version of the simple domain-dependent confounding example in Figure \ref{fig:dag}b in which $D$ is unobserved. We simulate data from linear Gaussian SEMs in which every variable is a linear combination of its parents plus Gaussian noise. Thus, every edge in the graph has a corresponding weight which is the coefficient of the parent in the SEM (for full specification consult the supplement). In particular, we let $C = w_2 D + \varepsilon_C$, $\varepsilon_C \sim \mathcal{N}(0, \sigma^2_C)$. We manifest domain-dependent confounding by varying $w_2$ in different test domains (from $w_2=2$ in the train domain), thus changing $p(C|D)$. 

First, note that in the ideal case if we could observe $D$ then the unstable edge $D\rightarrow C$ is not in any active paths to $T$ when $C$, $Y$ and $D$ are conditioned upon. This means a least squares regression modeling $E[T|C,Y,D]$ will have stable predictive performance regardless of changes to $w_2$ and $p(C|D)$. This is visible in Figure \ref{fig:lin-mse} in which the mean squared error (MSE) of the ideal model (red points) stays constant despite changes in $w_2$.

We could naively ignore changes to $p(C|D)$ and model by conditioning on vulnerable variables $E[T|C,Y]$, but the naive model's performance will vary in test domains as a result of using the unstable path. In fact, the MSE of the naive model (blue points in Figure \ref{fig:lin-mse}) appears to increase quadratically as $w_2$ changes from its training value.

Alternatively, we can use Counterfactual Normalization (CFN). $C$ and $Y$ are vulnerable, but we can condition on $Y(C=\emptyset)$.
First, we fit the structural equation for $Y$ from training data: $Y = \hat{w}_3 T + \hat{w}_4 C$.\footnote{$^\wedge$ denotes an estimated value.} We then estimate $\hat{Y}(\emptyset) = Y - \hat{w}_4C$. Finally, we model $E[T|\hat{Y}(\emptyset)]$ which conditions on a stable set. The MSE of CFN is stable (green points in Figure \ref{fig:lin-mse}), but it can be outperformed by the naive and ideal models because CFN isolates paths that include $D$.

\subsubsection{Cross Hospital Transfer}

\textbf{Ensuring Stable Performance}
\begin{table}[!t]
\centering
\caption{Simulated Experiment Results}
\label{table:sim-exp1}
\small
\begin{tabular}{|l|l|l|}
\hline
\textbf{Method}    & \textbf{Source AUROC} & \textbf{Target AUROC} \\ \hline
Baseline & 0.95                  & 0.80                  \\ \hline
CFN                & 0.96                  & 0.97                  \\ \hline
CFN (vuln)      & 0.97                  & 0.92                  \\ \hline
\end{tabular}
\vskip -0.2in
\end{table}

We consider a simulated version of the diagnosis problem in Figure \ref{fig:cf-dag}(a), but remove $X$ from the graph. We let $A$ represent the time since treatment and simulate the exponentially decaying effects of the treatment as $f(A) = 2\exp(-0.08A)$ where the treatment policy depends on $C$. $C$ and its descendants ($A$ and $Y$) are vulnerable. 

We simulate patients from two hospitals (full specification in the supplement). In the source hospital there is a positive correlation between $C$ and $T$, while in the target hospital $p(C|D)$ changes yielding a negative correlation. At the source hospital smaller $A$ are associated with $T=1$ while at the target hospital $A$ is uncorrelated with $T$. The structural equation for $Y$ remains stable: $Y = -0.5T  -0.3C + f(A) + \varepsilon_Y$, $\varepsilon_Y \sim \mathcal{N}(0, 0.2^2)$. We train using data from the source hospital and evaluate performance at both the source and target hospitals.

Counterfactual Normalization requires us to estimate the latent variable  $Y(A=\emptyset, C=\emptyset)$. We first fit the structural equation for $Y$ using maximum likelihood estimation, optimized using BFGS \citep{chong2013introduction}. Then, we compute the counterfactual: $Y_i(A=\emptyset, C=\emptyset) = Y_i - \hat{\beta} C_i - \hat{f}(s_i)$ for every individual $i$ at both hospitals, which can be done without observing $T$. We compare a counterfactual model (CFN) $p(T|Y(\emptyset, \emptyset))$ with a baseline vulnerable model $p(T|Y,A,C)$ and counterfactual model that uses vulnerable variables $p(T|Y(\emptyset, \emptyset), Y,A,C)$ using logistic regression and measure predictive accuracy with the area under the Receiver Operating Characteristic curve (AUROC). 

\begin{table}[!t]
\centering
\caption{Simulated Classification Complexity Metrics}
\label{table:metrics-sim-exp1}
\small
\begin{tabular}{|l|l|l|l|}
\hline
\textbf{Method} & \textbf{Fisher's} & \textbf{Distance} & \textbf{MST} \\ \hline
Baseline & 0.66                        & 0.10                    & 0.56                \\ \hline
CFN             & 3.13                        & 0.02                    & 0.22                \\ \hline
\end{tabular}
\vskip -0.1in
\end{table}
The results of evaluation on the patients from the source and target are shown in Table \ref{table:sim-exp1}. The accuracy of models that use vulnerable variables does not transfer across hospitals, with the baseline suffering large changes in performance. Instead, CFN transfers well while performing competitively at the source hospital, despite not using unstable paths which are informative in the training domain.

Normalizing BP ($Y$) for treatment ($A$) and chronic condition ($C$) greatly increases the separability by class in the training data as measured through the classification complexity metrics in Table \ref{table:metrics-sim-exp1}. The feature with the maximum Fisher's Discriminant Ratio in the baseline model is $C$, but this is much smaller than the ratio for the latent feature in CFN. The large decrease in the MST metric indicates fewer examples lies on the class boundary in the normalized problem, and the decrease in intraclass-interclass distance is due to a combination of increased separability and reduced intraclass variance of the latent variables. This is visible in the class conditional densities of factual and counterfactual $Y$ (see supplement).

\textbf{Accuracy of Counterfactual Estimates}
\begin{figure}[!t]
\begin{center}
\centerline{\includegraphics[scale=0.38]{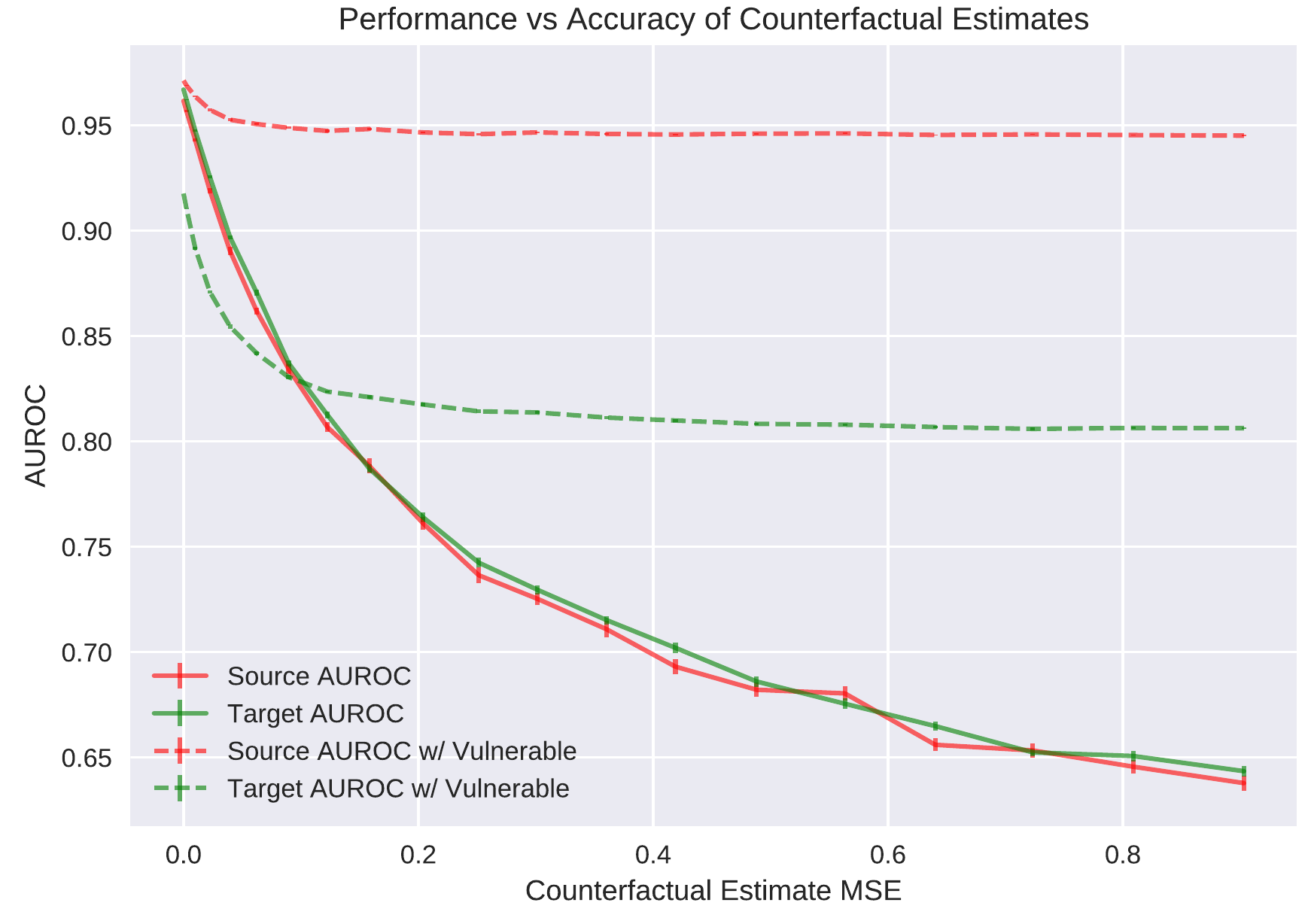}}
\caption{Performance as the accuracy of counterfactual estimates decreases.The error bars denote the standard error of 50 runs.}
\label{fig:exp1-perturb}
\end{center}
\vskip -0.4in
\end{figure}

In this experiment, we examine how the accuracy of counterfactual estimates affects model stability and performance. We expect models that do not use vulnerable variables to have more stable performance, but they may be less accurate in the source domain than models which use vulnerable variables. We bias the true counterfactual values by adding normally distributed noise of increasing scale. Then, we train the counterfactual logistic regressions (with ${p}(T|{Y}(\emptyset,\emptyset))$ and without ${p}(T|{Y}(\emptyset,\emptyset), C, A, Y)$ vulnerable variables) to predict $T$ and evaluate the AUROC at the source and target hospitals. We vary the standard deviation of the perturbations from $0.05$ to $1$ in increments of $0.05$, repeating the process 50 times for each perturbation. 

The results, shown in Figure \ref{fig:exp1-perturb}, confirm what we expect: removing vulnerable variables leads to more stable performance, but performance in the source domain is always lower than when including vulnerable variables. Further, when the counterfactual estimates are accurate (low MSE), removing vulnerable variables yields better performance in the target domain. However, when the MSE is high, the noise removes both the information captured by the adjustment and the information contained in $Y$ itself, causing the model to perform worse in the target domain than a model using vulnerable variables.

\subsection{REAL DATA: SEPSIS CLASSIFICATION}
\subsubsection{Problem and Data Description}
\begin{figure}[!t]
\begin{center}
\centerline{\includegraphics[scale=0.25]{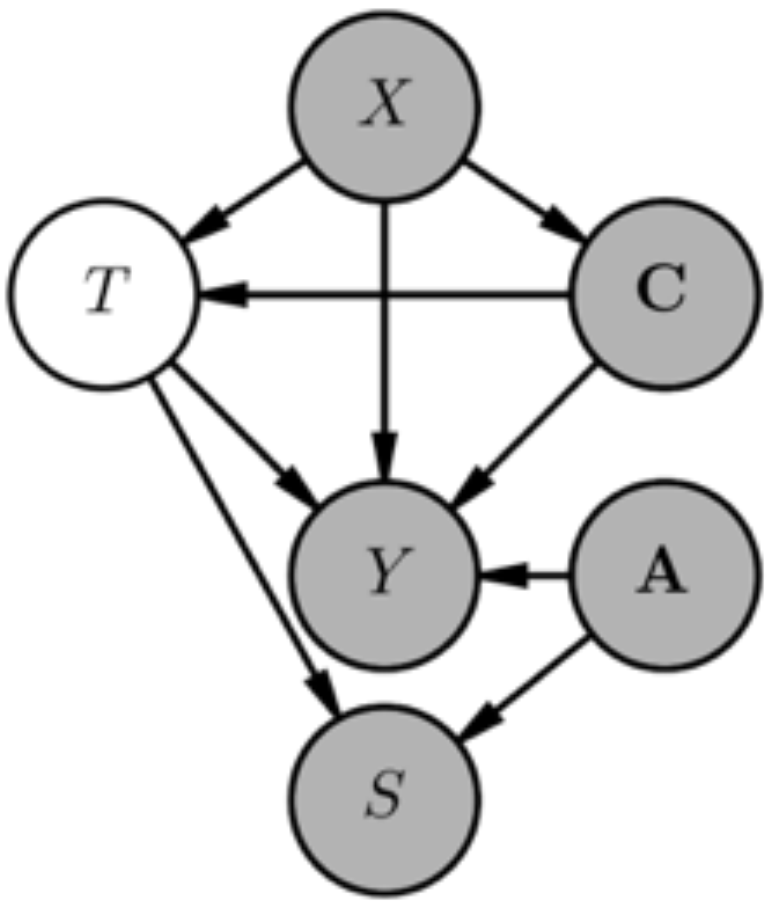}}
\caption{Real data experiment DAG of causal mechanisms. The outcome $Y$ is INR and the target $T$ is sepsis. $S$ represents selection bias.}
\label{fig:inr-dag}
\end{center}
\vskip -0.4in
\end{figure}

We apply the proposed method to the task of detecting sepsis, a deadly response to infection that leads to organ failure. Early detection and intervention has been shown to result in improved mortality outcomes \citep{kumar2006duration} which has resulted in recent applications of machine learning to build predictive models for sepsis (e.g., \citet{henry2015targeted,soleimani2017scalable,futoma2017ICML,futoma2017MLHC}). 

To illustrate, we consider a simplified\footnote{Sepsis involves many physiologic markers and corresponding treatments and chronic conditions. We select a small number of variables to demonstrate the key technical concepts.} cross-sectional version of the sepsis detection task using electronic health record (EHR) data from our institution's hospital. Working with a domain expert, we determined the primary factors in the causal mechanism DAG (Figure \ref{fig:inr-dag}) for the effects of sepsis on a single physiologic signal $Y$: the international normalized ratio (INR), a measure of the clotting tendency of blood. The target variable $T$ is whether or not the patient has sepsis due to hematologic dysfunction. We include seven conditions (such as chronic liver disease and sickle cell disease) $\mathbf{C}$ affecting INR that are risk factors for sepsis \citep{goyette2004hematologic,booth2010infection}. We consider five types of relevant treatments $\mathbf A$: anticoagulants, aspirin, nonsteroidal anti-inflammatory drugs (NSAIDs), plasma transfusions, and platelet transfusions, where $A_{ij}=1$ means patient $i$ has received treatment $j$ in the last 24 hours. Finally, we include a demographic risk factor, age $X$. For each patient, we take the last recorded measurements while only considering data up until the time sepsis is recorded in the EHR for patients with $T=1$.

27,633 patients had at least one INR measurement, 388 of whom had sepsis due to hematologic dysfunction. We introduced selection bias $S$ as follows. First, we took one third of the data as a sample from the original target population for evaluation. Second, we subsample the remaining data by rejecting patients with any treatment and without sepsis with probability $0.9$. Third, we split the subsampled data into a random two thirds/one third train/test splits for training on biased data and evaluating on both the biased and unbiased data to measure stability of prediction performance. We repeated the three steps 100 times. We normalize INR in all experiments.

\subsubsection{Experimental Setup}
We apply the proposed method by fitting an additive structural equation for $Y$ using the Bayesian calibration form of \citet{kennedy2001bayesian}:
\begin{align*}
    Y_i &= \beta_{0} + \beta_{1} T_i + \bm{\beta}_{2}^T \mathbf{A}_i + \bm{\beta}_{3}^T \mathbf{C}_i + \beta_{4} X_i\\
    &+ \delta(T_i, \mathbf{A}_i, \mathbf{C}_i, X_i) + \varepsilon\\
    \delta(&\cdot) \sim \mathcal{GP}(0, \gamma^2 K_{rbf})\\
    \varepsilon &\sim \mathcal{N}(0, \sigma^2)\\
\end{align*}
\vskip -8 mm
where $\delta(\cdot)$ is a Gaussian process (GP) prior (with RBF kernel) on the \emph{discrepancy function} since our linear regression model is likely misspecified. 

Due to selection bias and few sepsis examples, for better calibration we place informative priors on $\beta_1, \mathbf{\beta}_2,$ and $\mathbf{\beta_3}$ using $\mathcal{N}(1, 0.1)$ for features that increase INR (e.g., $T$ and anticoagulants) and $\mathcal{N}(-1, 0.1)$ for features that decrease INR (e.g., sickle cell disease and plasma transfusions). For full specification of the other priors  consult the supplement. We compute point estimates for the parameters using MAP estimation and the FITC sparse GP \citep{snelson2006sparse} implementation in PyMC3 \citep{salvatier2016probabilistic}.

While the only vulnerable variables are $\mathbf{A}$ and $Y$, we additionally remove the effects of $\mathbf{C}$ and $X$:
\begin{equation}
{Y_i}(\bm{\emptyset},\bm{\emptyset}, \emptyset) = Y_i -  \hat{\bm{\beta}}_{2}^T \mathbf{A}_i - \hat{\bm{\beta}}_{3}^T \mathbf{C}_i - \hat{\beta}_4 X_i 
\end{equation}
We consider three logistic regression models trained on the biased data for predicting $T$: a baseline using vulnerable variables $p(T|\mathbf{A},\mathbf{C}, Y, X)$, a counterfactually normalized model $p(T|\mathbf{C}, Y(\bm{\emptyset},\bm{\emptyset}, \emptyset), X)$, and a counterfactually normalized model with vulnerable variables $p(T|\mathbf{C}, Y(\bm{\emptyset},\bm{\emptyset}, \emptyset), Y, X)$. We evaluate prediction accuracy on biased and unbiased data using AUROC and the area under the precision-recall curve (AUPRC).

\subsubsection{Results}
\begin{figure}[!t]
\begin{center}
\centerline{\includegraphics[scale=0.4]{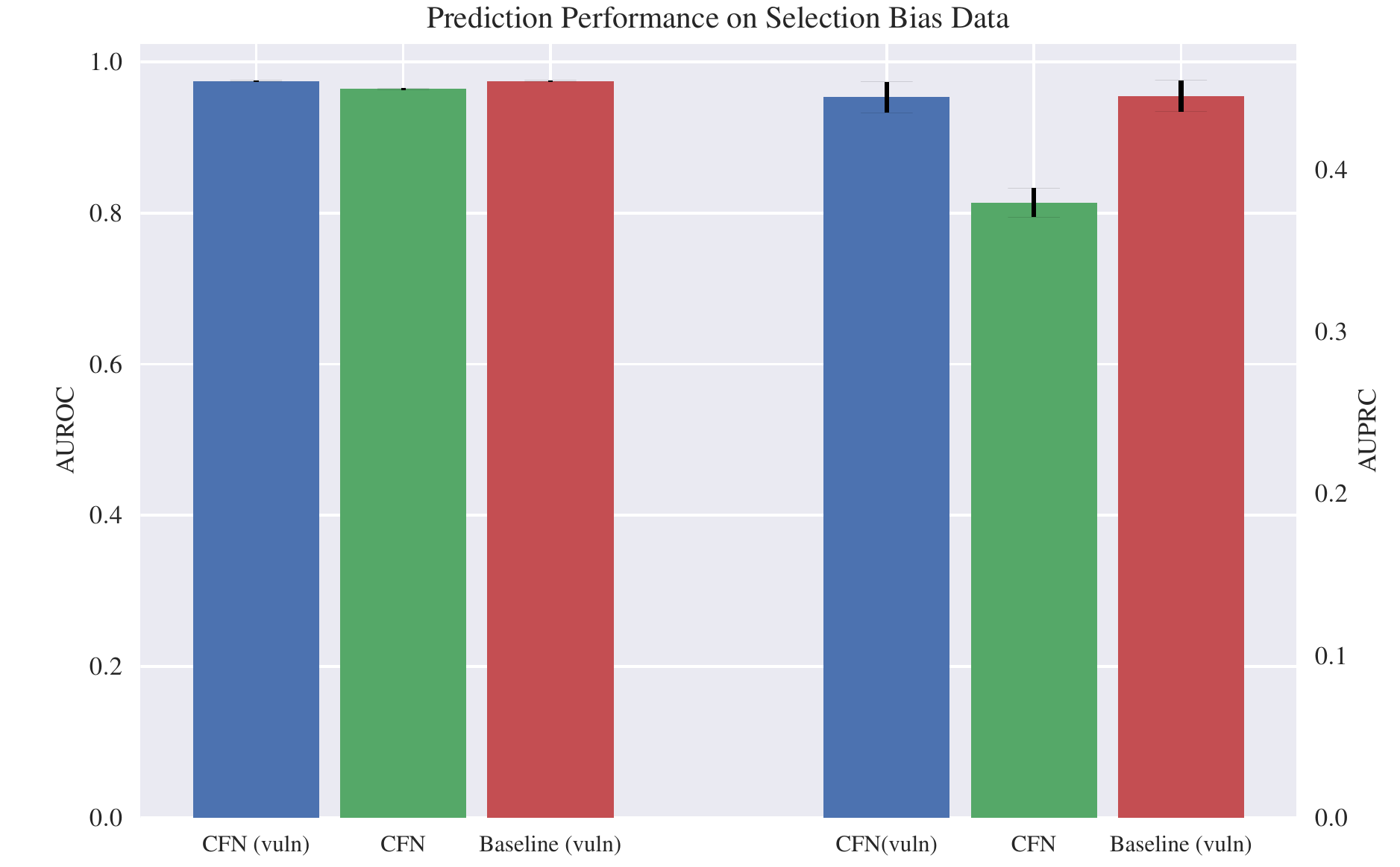}}
\caption{Results for models trained and tested on the selection biased data. In order the average AUROCs are 0.98, 0.96, and 0.98 and the average AUPRCs are 0.45, 0.38, and 0.45. Error bars denote 100 run $95\%$ intervals.}
\label{fig:biased-AUC}
\end{center}
\vskip -0.2in
\end{figure}

\begin{figure}[!t]
\begin{center}
\centerline{\includegraphics[scale=0.4]{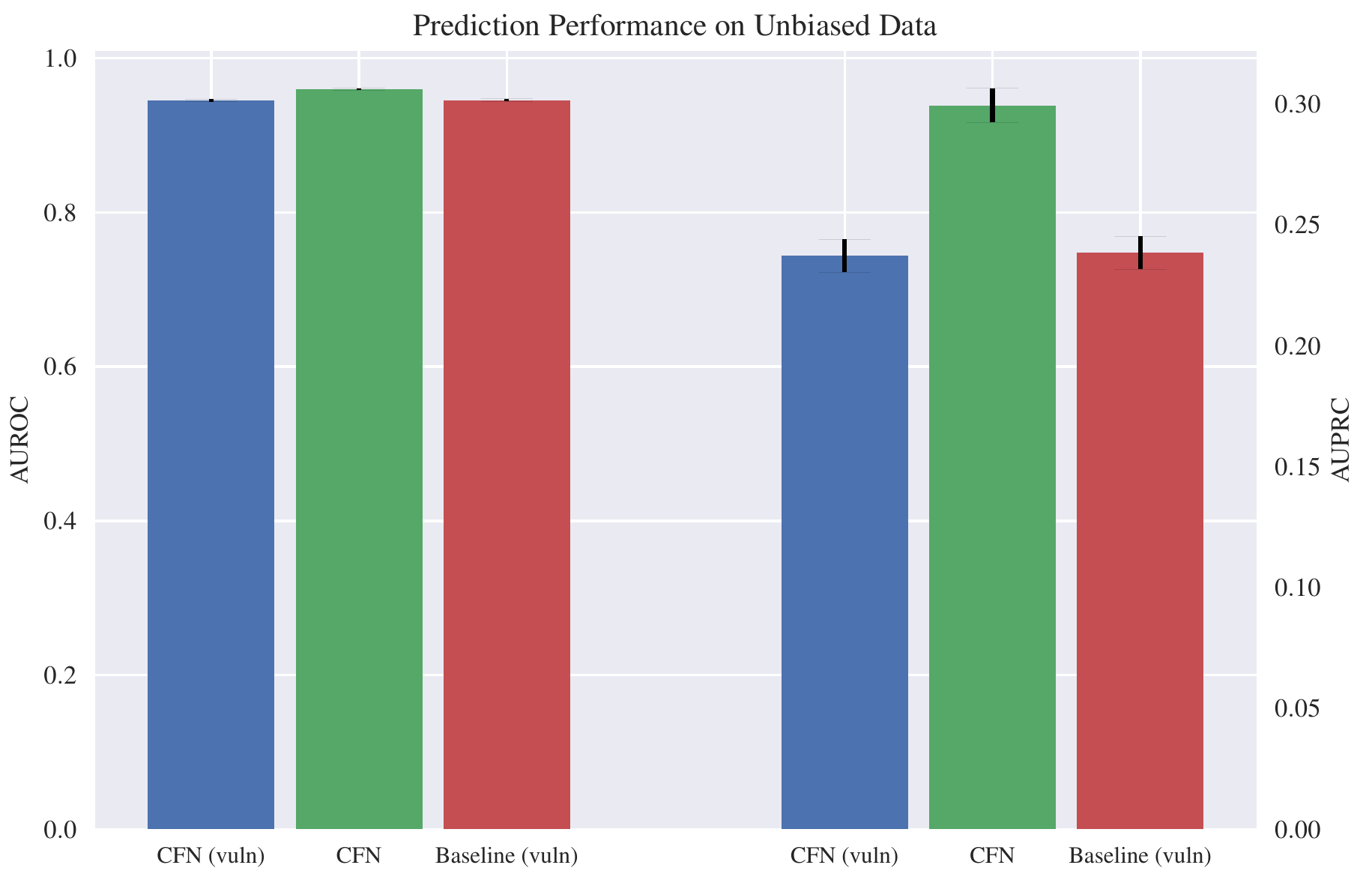}}
\caption{Results for models trained on biased data and tested on unbiased data. In order the average AUROCs are 0.95, 0.96, and 0.95 and the average AUPRCs are 0.24, 0.30, and 0.24. Error bars denote 100 run $95\%$ intervals.}
\label{fig:unbiased-AUC}
\end{center}
\vskip -0.4in
\end{figure}
The selection bias causes a small shift in the marginal distribution of $T$ between populations, such that $2\%$ of the selection biased population has sepsis while $1.4\%$ of the unbiased population has sepsis.
Since most of the examples are negative, the AUPRC is a more interesting measurement because it is sensitive to false positives.

The resulting AUCs when predicting on selection biased data are shown in Figure \ref{fig:biased-AUC}. As expected, the counterfactually normalized model (CFN) performs worse than models using vulnerable variables because it does not take advantage of the unstable path created by selection bias. On unbiased data (Figure \ref{fig:unbiased-AUC}), however, CFN not only outperforms both vulnerable models, but its performance is also more stable to the selection bias: the decrease in AUPRC from source to target is much smaller for CFN.

Interestingly, the the performance of the two vulnerable models is nearly identical. This implies that the CFN model with vulnerable variables does not learn to use counterfactual features, perhaps because the unstable path through selection bias encodes a much stronger relationship. The AUPRC of the non-vulnerable CFN model in selection biased and unbiased data is in between the AUPRC of the vulnerable CFN model in selection biased data (upper bound) and unbiased data (lower bound). This is encouraging, because Figure \ref{fig:exp1-perturb} suggests that if CFN performance were worse in the unbiased data than the vulnerable model's performance, then the counterfactual estimates may be inaccurate. Ultimately, we were able to leverage Counterfactual Normalization to remove vulnerable variables resulting in stabler performance, while outperforming the vulnerable models in unbiased data.

%% file: sections/conclusion.tex
When environment-specific artifacts cause training and test distributions to differ, naively training models under i.i.d. assumptions can result in unreliable models which predict using unstable relationships that do not generalize. While some previous solutions use prior knowledge of causal mechanisms to predict potential outcomes that are invariant to differences in policy across environments, they require strong assumptions about no unobserved confounders that may not hold in practice (e.g., \citet{schulam2017NIPS}). Our proposed solution, Counterfactual Normalization, generalizes these approaches to cases in which the unstable relationships (such as ones due to domain-specific policy) may depend on unobserved variables or selection bias. Specifically, we train discriminative models using conditioning sets that only contain variables with stable relationships with the target prediction variable. Then, for vulnerable variables with unstable relationships to the target, we consider adding to the conditioning set counterfactual versions of these variables which sever the unstable paths of statistical influence. Further, because of their causal interpretations, we believe these counterfactual variables are more intelligible for human experts than existing adjustment-based methods. For example, we think it is easier to reason about ``the blood pressure if the patient had not been treated'' than interaction features or kernel embeddings---we would like to test this in a future user study. As demonstrated by our experiments, models trained using Counterfactual Normalization have performance that is more stable to changes across environments and is not coupled to artifacts in the training domain.

%% file: sections/appendix.tex
\section{Counterfactual Normalization Proofs}
\subsection{Proof of Theorem 1}
We must show that on iteration $k$, removing a variable from $\mathbf{Z}$ does not create an active unstable path to a member of $\mathbf{Z}$ of length $\leq k$.

\begin{proof} Suppose, by contradiction, that on iteration $k$ removing a variable $v\in\mathbf{Z}$ with an active unstable path of length $k$ to $T$ results in an active unstable path of length $\leq k$ with respect to another variable $x\in\mathbf{Z}$.
Note: removing a variable from a conditioning set cannot create new collider paths. Let $\dots$ denote that direction of edge does not matter. We will consider all cases of how $v$ can relate to an unstable path to $x$ from $T$. In the first two cases, $x$ comes before $v$ in the unstable active path to $v$.

Case 1: $T \dots x \dots unstable \dots v$. $x$ does not have an unstable path to $T$. If it did, the path would be of the form $T\dots unstable\dots x \dots unstable \dots v$. Thus, $x$ would have been removed from $\mathbf{Z}$ in a previous iteration because the unstable path is of length $\leq k$ and its active status does not depend on $v$.

Case 2: $T \dots unstable \dots x \dots v$. This is an unstable path to $x$ of length $\leq k$. $x$ cannot be in $\mathbf{Z}$ since it would have been removed in a previous iteration as the active status of this path does not depend on $v$.

Case 3: $T \dots unstable \leftarrow v \rightarrow \dots x$. Creates new active path of length $>k$.

Case 4: $T \dots unstable \rightarrow v \rightarrow \dots x$. Creates new active path of length $>k$.

Case 5: $T \dots unstable \leftarrow v \leftarrow \dots x$. Creates new active path of length $>k$.

Case 6: $T \dots unstable \rightarrow v \leftarrow \dots x$. We remove $v$ from the conditioning set (so it is now considered unobserved). Thus, this collider path is not active. If a descendent of $v$ is conditioned on, then this is an unstable active path of length $> k$.

In all cases, either $x$ would have been removed from $\mathbf{Z}$ before iteration $k$ or the new unstable active path would be of length $>k$. This is a contradiction since we assumed $x\in\mathbf{Z}$ and that the procedure would result in a new active unstable path of length $\leq k$.
\end{proof}

\subsection{Proof of Theorem 2}
We must show that Algorithm \ref{alg:CFN} will not activate any unstable paths with respect to the initial stable set $\mathbf{Z}$. While considering each vulnerable variable $v\in\mathbf{V}$, the resulting set $\mathbf{Z'}$ must remain stable.

\begin{proof}
We assume the initial set $\mathbf{Z}$ is stable. The only way to activate a path (stable or unstable) by adding to a conditioning set is if the new variable being added is a collider or descendant of a collider.

Algorithm \ref{alg:CFN} only adds to $\mathbf{Z}$ in branches 2 and 3 of the if-else. We consider each branch in turn.

In branch 2, the vulnerable variable $v\in\mathbf{V}$ has no active unstable paths to $T$. Thus, by adding $v$ to $\mathbf{Z}$, no unstable path from $v$ to $T$ is active. Next we consider the possibility of $v$ being a collider or descendant of a collider. In these cases, if $v$ activates a path all branches of the activated path $\dots \rightarrow collider \leftarrow \dots$ will be reachable by $v$. Since $v$ has no unstable active paths to $T$, none of the branches of the collider can be unstable paths to $T$. Thus, the collider path is not unstable.

In branch 3 of the if-else, the active unstable paths of the vulnerable variable $v\in\mathbf{V}$ all go through some subset of the observed parents of $v$. Denote this subset of parents of $v$ as $\mathbf{P}\subseteq pa(v)$. We intervene on $\mathbf{P}$ and node-split to add the counterfactual $v(\mathbf{P}=\emptyset)$ to the modified graph. The modified graph contains the path $v(\mathbf{P}=\emptyset) \rightarrow v \leftarrow \mathbf{P}$. Conditioning on $v(\mathbf{P}=\emptyset)$ does not allow for paths through the collider $v$. We know the paths through the parents of $v(\mathbf{P}=\emptyset)$ are stable because we intervened on all parents on unstable paths from $v$. Further, there are no active unstable paths through children of $v$ otherwise $v$ would not be considered in branch 3. As in the case of branch 2, if $v(\mathbf{P}=\emptyset)$ is a collider or descendant of a collider, then since all parts of any activated collider path are reachable from the counterfactual variable through its parents, they must be stable.

Thus, when Algorithm \ref{alg:CFN} adds a variable or counterfactual variable to the stable set $\mathbf{Z}$, no unstable paths are activated and the set remains stable.
\end{proof}

\section{Linear Gaussian Experiment Details}\label{sec:appendix}
\subsection{Simulation Details}
We generate the data from the following linear Gaussian SEMs:
\begin{align*}
    &D = \varepsilon_D\\
    &T = w_1 D + \varepsilon_T\\
    &C = w_2 C + \varepsilon_C\\
    &Y = w_3 T + w_4 C + \varepsilon_Y\\
    &\varepsilon_D,\varepsilon_T,\varepsilon_C,\varepsilon_Y\sim \mathcal{N}(0, 0.1^2)
\end{align*}
In the training domain, we set $w_2=2$ and $w_1,w_3,w_4 \sim \mathcal{N}(0, 1)$. We simulated 100 test domains by varying $w_2$ from $-3$ to $7$ in equally spaced increments. In all domains we generated 30000 samples. We fit all models (structural equation of $Y$, naive model, ideal model, and counterfactually normalized model) on the training data using least squares then applied them to all test domains.
\subsection{Counterfactual Normalization}
\begin{figure}[!h]
\begin{center}
\centerline{\includegraphics[scale=0.48]{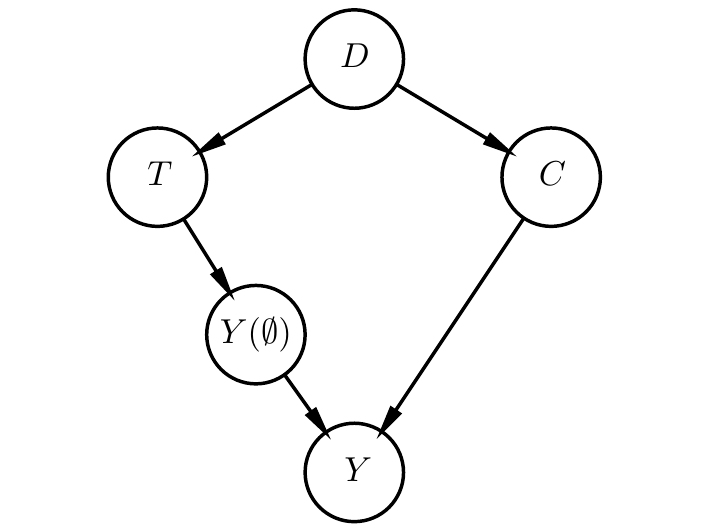}}
\caption{Modified graph after node-splitting.}
\label{fig:lin-sem-node-split}
\end{center}
\vskip -0.2in
\end{figure}
$D$ is unobserved and $p(C|D)$ varies between domains. Thus, $C$ is vulnerable becomes conditioning on $C$ results in the unstable active path $C\leftarrow D \rightarrow T$. $Y$, as a descendant of $C$, is also vulnerable because without conditioning on $C$, conditioning on $Y$ results in the unstable active path  $Y \leftarrow C \leftarrow D \rightarrow T$. The only shared child of vulnerable variables and $T$ is $Y$. This means we need to perform node-splitting on $Y$ to generate an intermediate counterfactual version, $Y(C=\emptyset)$ for which we have removed the effects of the vulnerable parent $C$. The graph after node-splitting is shown in Figure \ref{fig:lin-sem-node-split}.

The counterfactual variable $Y(C=\emptyset)$ inherits the parents of $Y$ from the original graph that we did not intervene upon (i.e., set to null). In this case, the parents it inherits are $T$ and the unpictured $\varepsilon_Y$. The SEMs in the modified graph are:
\begin{align*}
    &D = \varepsilon_D\\
    &T = w_1 D + \varepsilon_T\\
    &C = w_2 C + \varepsilon_C\\
    &Y(C=\emptyset) = w_3 T + \varepsilon_Y\\
    &Y = Y(C=\emptyset) + w_4 C\\
    &\varepsilon_D,\varepsilon_T,\varepsilon_C,\varepsilon_Y\sim \mathcal{N}(0, 0.1^2)
\end{align*}
Importantly, note that the counterfactual $Y(C=\emptyset)$ is now a random quantity, while $Y$ is a \emph{deterministic} function of $Y(C=\emptyset)$ and $C$. The modified SEMs are observationally equivalent to the original SEMs (marginalizing over the latent counterfactual yields the same joint as in the original system). 

We can recover $Y(C=\emptyset)$ by observing $Y$ and $C$: $Y(C=\emptyset) = Y - w_4 C$.
This makes clear the role of the intermediate counterfactual variable: it isolates the effect of the target on the outcome from the effects of vulnerable variables (or any other parents we set to null) on the outcome.

\section{Cross Hospital Transfer Experiment Details}
\subsection{Simulation Details}
We generate data at the source hospital as follows:
\begin{align*}
    &D \sim Bernoulli(0.5)\\
    &T|D=1 \sim Bernoulli(0.7)\\
    &T|D=0 \sim Bernoulli(0.1)\\
    &C|D=1 \sim Bernoulli(0.9)\\
    &C|D=0 \sim Bernoulli(0.1)\\
    &A|C=1 \sim 24 * Beta(0.5, 2.1)\\
    &A|C=0 \sim 24 * Beta(0.7, 0.2)\\
    &Y \sim \mathcal{N}(-0.5T + -0.3C + f(A), 0.2^2)\\
    &f(A) = 2\exp(-0.08A)
\end{align*}

At the target hospital, we change $p(C|D)$ and $p(A|C)$:
\begin{align*}
    &D \sim Bernoulli(0.5)\\
    &T|D=1 \sim Bernoulli(0.7)\\
    &T|D=0 \sim Bernoulli(0.1)\\
    &C|D=1 \sim Bernoulli(0.1)\\
    &C|D=0 \sim Bernoulli(0.9)\\
    &A|C=1 \sim 24 * Beta(1.7, 1.1)\\
    &A|C=0 \sim 24 * Beta(1.7, 1.1)\\
    &Y \sim \mathcal{N}(-0.5T + -0.3C + f(A), 0.2^2)\\
    &f(A) = 2\exp(-0.08A)
\end{align*}

We generate 2000 patients from the source hospital, using 1600 for training and holding out 400 to evaluate performance on the source hospital. We evaluate cross hospital transfer on 1000 patients generated from the second hospital.

\subsection{Class Conditional Densities}
\begin{figure}[h]
\begin{center}
\centerline{\includegraphics[scale=0.43]{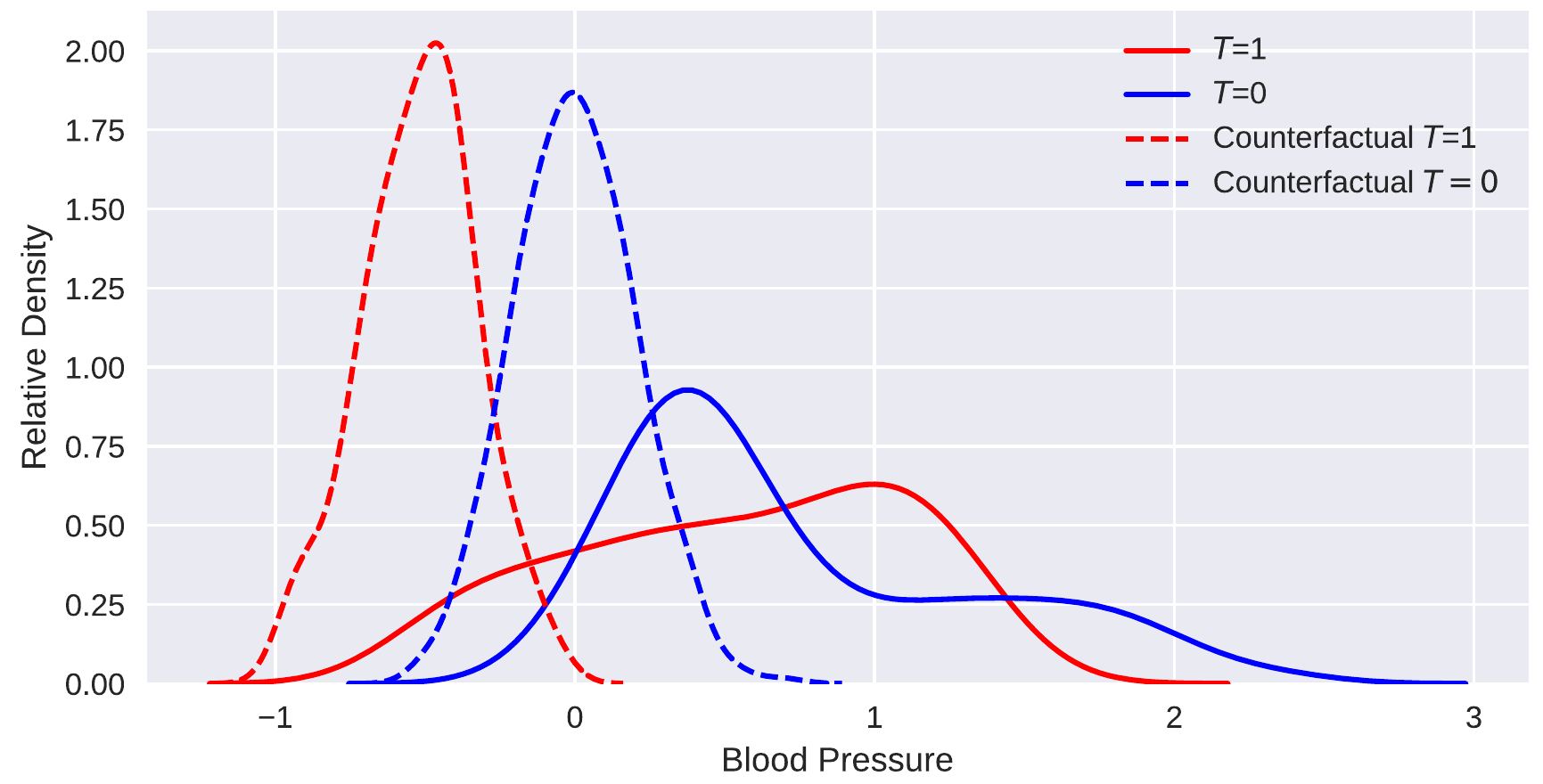}}
\caption{The distribution of factual (solid line) and estimated counterfactual (dashed line) blood pressures at the source hospital in the simulated experiment. It is easier to discriminate $T$ from counterfactual BP than from observed BP due to decreased overlap in the distributions.}
\label{fig:exp1-bp-density}
\end{center}
\end{figure}

\section{Real Data Experiment Details}\label{sec:appendix-real-data}
Our posited structural equation for INR ($Y$) is a linear regression of the parents of $Y$ in Figure 5. The seven conditions ($\mathbf{C}$) we include are liver disease, sickle cell disease, chronic kidney disease, any immunodeficiency, any cancer, diabetes, and stroke. In the statistical uncertainty quantification community, one technique for parameter calibration when the computer model is misspecified is to jointly estimate model parameters with an explicit discrepancy function that captures model inadequacy (Kennedy and O'Hagan, 2001). The discrepancy function has a Gaussian process prior. The parameters to estimate are the linear regression parameters $\bm \beta$, the observation noise scale $\sigma$, the RBF kernel output scale $\gamma$, and the kernel lengthscales $\bm \ell$.

We placed the following priors on parameters:
\begin{align*}
    &\gamma \sim Half\mathcal{N}(1) \\
    &\sigma \sim Half\mathcal{N}(1) \\
    &\bm \ell \sim Gamma(4, 4) \\
    &\beta_0 \sim \mathcal{N}(0, 1)\\
    &\beta_1, \bm{\beta}_2^{(1, 2, 3, 5)}, \bm{\beta}_3^{(1)} \sim \mathcal{N}(1, 0.1)\\
    &\bm{\beta}_2^{(4)}, \bm{\beta}_3^{(2)} \sim \mathcal{N}(-1, 0.1)\\
    &\beta_4 \sim \mathcal{N}(0, 0.1)
\end{align*}
We used the PyMC3 FITC sparse GP approximation implementation with 20 inducing points initialized by k-means.